\documentclass{article}

\PassOptionsToPackage{numbers,compress}{natbib}
\usepackage[preprint]{neurips_2026}
\workshoptitle{Transitioning from Pre-Training to Post-Training}

\usepackage[utf8]{inputenc}
\usepackage[T1]{fontenc}
\usepackage{hyperref}
\usepackage{url}
\usepackage{booktabs}
\usepackage{multirow}
\usepackage{fontawesome5}
\usepackage{amsmath,amssymb,amsfonts,mathtools}
\usepackage{nicefrac}
\usepackage{microtype}
\usepackage{xcolor}
\usepackage{graphicx}
\usepackage{xspace}
\usepackage{amsmath}
\usepackage{bm}
\usepackage{amsthm} 
\newtheorem{lemma}{Lemma} 
\usepackage{wrapfig}
\usepackage{caption}
\usepackage{subcaption}

\newcommand{\ours}{3vKCCA\xspace}

\title{{Revisiting  Visual Representation Enhancement of VLMs via Kernel Canonical Correlation Analysis}}

\author{
  Peilin Yang\thanks{These authors contributed equally.}
  \quad
  Xiaoyu Liu\footnotemark[1]
  \quad
  Jian Sun
  \quad
  Qinghua Tao\thanks{Corresponding author. Email: qinghua.tao@bit.edu.cn}
  \\
  Beijing Institute of Technology
}

\begin{document}
\maketitle

\begin{abstract}
Vision-language models such as CLIP exhibit strong semantic generalization, but remain limited in fine-grained visual perception. A recent work named KUEA presents a natural remedy by finetuning the image encoder under the supervision of the vision-centric DINOv2 to align their kernel matrices element-wisely, while regularizing the embeddings to remain close to the pretrained  visual encoder for preserving image-text semantics in CLIP. However, we show that diminishing the role of the alignment loss to DINOv2 does not necessarily degrade its fine-grained visual performance, suggesting that the kernel-matrix discrepancy may  be insufficient for further visual representation enhancement,  motivating us to revisit the alignment formulation.  In this work, we  present a novel perspective to characterize  representation alignment on feature subspaces through Kernel Canonical Correlation Analysis (KCCA), which maximizes the projection correlations. In optimization, we derive an efficient end-to-end training scheme upon  KKT conditions, avoiding the eigenvalue problem in KCCA. Further, we extend our method into a 3-view formulation, i.e., \ours, in which the projections from the pretrained text encoder are also incorporated under a unified optimization framework for joint alignment. With CLIP ViT-L/14 on ImageNet-1K, our \ours  improves the MMVP-VLM accuracy from $17.8$ to $25.9$, substantially outperforming the existing methods, and meanwhile maintains zero-shot image--text retrieval performance.

\end{abstract}

\section{Introduction}
Vision–Language Models (VLMs) have been attracting tremendous attention in recent years thanks to their strong semantic generalization in various downstream tasks ~\cite{li2022languagedriven, minderer_eccv_2022,Saha_2024_CVPR}. Notably, the Contrastive Language-Image Pretraining (CLIP) model~\cite{radford2021clip} demonstrates strong zero-shot recognition capabilities by aligning visual representations with textual descriptions. However,  such contrastive pretraining
does not necessarily encourage the visual encoder to capture finer pixel-level details, such as spatial relations, even when these details are readily distinguishable to human observers~\cite{tong2024eyes}.

Vision-only self-supervised models such as DINOv2~\cite{oquab2024dinov2} focus on visual representations that capture fine  details with strong visual grounding, making them suitable supervisors for refining the visual encoder in CLIP. {A key challenge naturally arises: \emph{how to enhance the visual representation without downgrading the  image--text alignment of the pretrained VLM model.}  A recent work, KUEA~\cite{gong2025kuea}, proposes to balance these two goals with a weighted combination objective.  KUEA minimizes an $\ell_2$-norm  loss for visual alignment between two kernel matrices computed  from embeddings of the finetuned visual encoder and  DINOv2, and also regularizes an $\ell_2$-norm term to confine the finetuned visual embeddings close to the pretrained counterparts in CLIP for image-text pairing semantics.}


{We revisit the kernel-alignment loss in KUEA by examining a premise underlying this formulation: \emph{whether and how the fine-grained visual capability benefits from the minimization of such alignment loss}. We examine this premise on  MMVP-VLM~\cite{tong2024eyes}, a benchmark for assessing fine-grained visual abilities.  In our reproduction, removing the kernel-alignment loss from KUEA yields similar MMVP-VLM accuracy as the full objective; when further minimizing the alignment loss and meanwhile downweighting the regularization term,  the fine-grained visual performance does not improve either}. 
{This observation indicates that \emph{(i)} the continual training  with the regularization term may also contribute to enhancing visual capabilities, and \emph{(ii)} the optimization upon this kernel-alignment loss with DINOv2 can be insufficient  for further enhancing fine-grained visual capabilities. We notice that a concurrent work, namely KALE~\cite{pawlowicz2026kale}, partially addresses the above issues through a loss-equilibration controller that  rescales the kernel-alignment loss towards a desired ratio over the regularization term for pronouncing the alignment to DINOv2. We go beyond the loss of straightforwardly approaching the kernel matrix from DINOv2, and also reconsider the maintenance of image-text  semantics in the pretrained CLIP from the role of the text encoder under a unified framework.}
{In this work, we present a novel perspective that characterizes the representation alignment on feature subspaces through  Kernel Canonical Correlation Analysis (KCCA) \cite{bach2002kernel,suykens2002lssvm}. Under our framework,  the alignment is formulated by the maximal correlations between the representation projections of the finetuned visual encoder and DINOv2. KCCA has long been considered a fundamental tool to identify projection directions, along which cross-view variations are maximally associated. The couplings in KCCA are celebrated with explicit optimization formulas and theoretical guarantees \cite{bach2002kernel,hotelling1936relations},  in contrast to the existing kernel-alignment loss of straightforwardly minimizing element-wise distances. Furthermore, we resort to  the projections of the pretrained text encoder from  CLIP and  incorporate them into the alignment under our KCCA-based framework with a 3-view  extension, namely \ours. Then, a joint optimization problem is established, which finetunes the  image encoder of CLIP to align  not only with the supervision of DINOv2 for enhancing fine-grained visual  capabilities, but also  with the frozen  text encoder from CLIP for preserving image--text semantics.}

Experiments on OpenAI CLIP ViT-L/14 show that our \ours distinctively improves fine-grained visual capabilities on MMVP-VLM from 17.8 to 25.9, while preserving strong zero-shot image--text retrievals of the pretrained CLIP. {We then  further investigate  our perspective upon  subspace projection correlations for representation alignment. Firstly,  ablation studies show that our performance gain diminishes when  the KCCA-alignment loss to DINOv2 is removed, when the correspondence to DINOv2 is shuffled, or when  DINOv2 is replaced by a random teacher model for supervision; these results verify 
that the fine-grained visual capabilities indeed benefit from our alignment formulation. Secondly, our 3-view KCCA extension leveraging the frozen text encoder consistently outperforms the 2-view alternative; this reveals that upon aligning to  strong visual supervision, visual representations are promising to be enhanced via cross-view couplings  complemented by text semantics.}

\section{Related Work}
\label{sec:related_work}

\paragraph{Enhancing Visual Representations of CLIP.}
{Prior work considers the enhancement of fine-grained visual representations of VLMs mainly by localized supervision or vision-centric teacher models.}
RegionCLIP learns from region--text pairs~\cite{zhong2022regionclip}, whereas FG-CLIP combines detailed captions, region annotations, and hard negatives~\cite{xie2025fgclip}. CLIP-DINOiser transfers DINO-derived localization priors to CLIP~\cite{wysoczanska2024clipdinoiser}, and Talk2DINO maps CLIP text embeddings to DINOv2 patch features~\cite{barsellotti2025talk2dino}. {More closely related to image-only finetuning as addressed in this work, DIVA uses diffusion feedback~\cite{wang2025diva}, while KUEA matches the elements of kernel matrices from CLIP and DINOv2~\cite{gong2025kuea}. Upon KUEA, KALE adaptively balances the reweighted loss terms for improved fine-grained visual capabilities~\cite{pawlowicz2026kale}.}

\paragraph{Kernel CCA and its Multi-View Extensions.}
Canonical correlation analysis (CCA) \cite{hotelling1936relations} learns linear projections that maximize correlation between a pair of views, i.e., two sources of data.  Kernel CCA (KCCA) extends this principle to nonlinear relations through kernel tricks~\cite{bach2002kernel,suykens2002lssvm,vangelstel2001kcca}. Deep CCA \cite{andrew2013dcca} parameterizes the nonlinear  transformations with neural networks 
and  \cite{chapman2024stochasticcca} further derives generic deep CCA formulations together with their multi-view extensions. 
{Some recent work considers connecting  CCA to the training objectives of Barlow Twins~\cite{zbontar2021barlow} and VICReg~\cite{bardes2022vicreg} in large language models.
Our work  leverages  Least Squares Support Vector Machines (LSSVM) \cite{suykens2002lssvm} to formulate a KCCA-based objective over 3 functionally distinct views, where  the stationary solutions with their  Karush–Kuhn–Tucker (KKT) conditions ensure a zero-value objective in an unconstrained optimization form, making it possible to be trained in an efficient end-to-end manner.}

\section{Method}
\label{sec:method}


{In this section, we  innovate to characterize
the representation alignment through their projection correlations in feature subspaces under the KCCA framework to improve visual capabilities in CLIP. This perspective differs from the {mainstream approach} of minimizing sample-wise distances, e.g., the difference between the element-wise discrepancy of two kernels in KUEA or even the contrastive loss in pretraining CLIP. Our work resorts to the alignment through the correlations in projected subspaces, in which  maximal information couplings are guaranteed by the well-established KCCA method. An efficient end-to-end training scheme under KKT conditions  is then proposed.}

\paragraph{Notation.}  {Consider a pretrained CLIP consisting of  an image encoder $f_{\rm{img}}(\cdot)$ with  parameters $\bm \theta_0$ and a text encoder $f_{\rm{txt}}(\cdot)$, and   DINOv2  $f_{\rm{dino}}(\cdot)$. Let $\hat{f}_{\rm{img}}(\cdot)$ denote the finetuned image encoder of CLIP with  parameter $\bm \theta$. Given image--text pairs $\{(I^{(n)},T^{(n)})\}_{n=1}^{N}$, we use $\mathcal{V}=\{V,D,T\}$ as subscripts to index the  views w.r.t. the  finetuned CLIP visual encoder  $\hat{f}_{\rm{img}}$, DINOv2 $f_{\rm{dino}}$, and the fixed  text  encoder of CLIP  $f_{\rm{txt}}$, respectively, with their corresponding embeddings as $\bm h_v$, $v\in \mathcal{V}$.}


\begin{figure}[t]
    \centering
    \includegraphics[
        width=1\textwidth,
        trim=5 70 45 35,
        clip
    ]{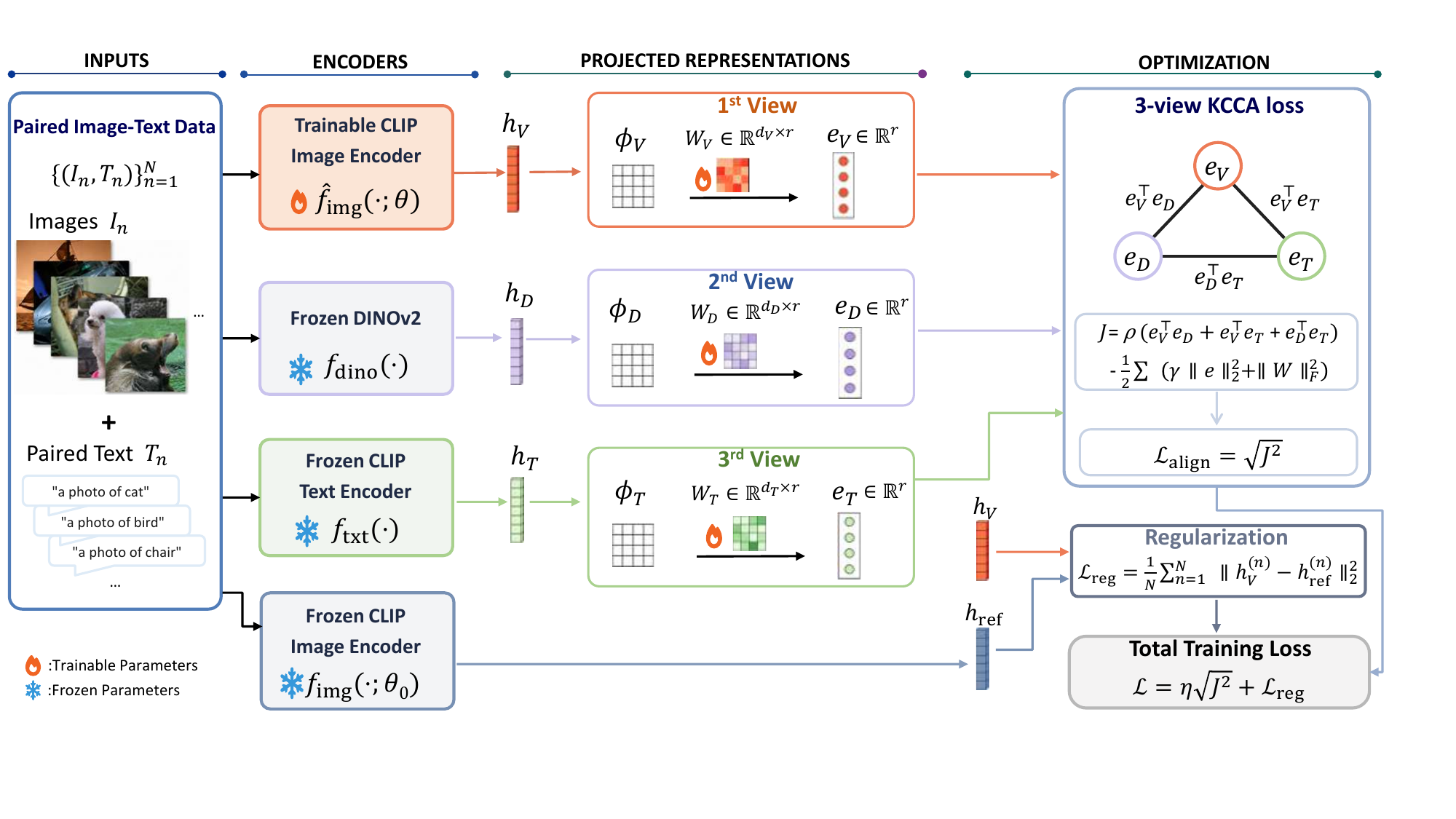}
    \caption{
    \textbf{Overall schematics of \ours.}
    The CLIP vision encoder is jointly aligned with the given DINOv2 and the frozen CLIP text encoder via 3-view KCCA, together with regularization.}
    \label{fig:method_overview}
\end{figure}

\subsection{Aligning Subspace Projections with Kernel CCA}
\label{sec:subspace_cca}

\paragraph{KCCA-Alignment with DINOv2 Supervision.}
We cast the representation alignment between the fine-tuned CLIP image 
encoder and the frozen DINOv2 encoder under the KCCA framework. Denote $\phi_V(\cdot)$ and $\phi_D(\cdot)$ as the  feature mappings w.r.t. CLIP-image and DINOv2 embeddings in KCCA, respectively, where their inner products can be evaluated implicitly through  kernel tricks $\kappa_{V}(\cdot,\cdot)=\langle\phi_V(\cdot), \phi_V(\cdot)\rangle$ and $\kappa_{D}(\cdot,\cdot)=\langle\phi_D(\cdot), \phi_D(\cdot)\rangle$, allowing to capture nonlinear correlations.
For  views $v\in\{V,D\}$, we denote
$W_v=[\bm w_{v,1},\ldots,\bm w_{v,r}]$ as the corresponding projection matrices
 in the $r$-dimensional subspaces of KCCA and  $\bm e_v=W_v^\top\phi_v\in\mathbb{R}^r$ as the resulting projection scores, which are maximally correlated in the optimization problem. In this work, we formulate KCCA under
  LSSVM setups \cite{suykens2002lssvm}, leading to the following  primal optimization:
\begin{equation}\label{eq:cca:lssvm}
\begin{array}{rll}
    \max\limits_{W_V, W_D, \bm e_V, \bm e_D}& -{\rm{Tr}}(W_V^\top  W_V) - {\rm{Tr}}(W_D^\top  W_D) +   \dfrac{1}{\lambda}\sum\limits_{n=1}^N \bm e_V^{(n)}\bm e_D^{(n)} - \dfrac{{\rho_1}}{2}\sum\limits_{n=1}^N {\bm e_V^{(n)}}^2 - \dfrac{{\rho_2}}{2}\sum\limits_{n=1}^N {\bm e_D^{(n)}}^2 
  \\
{\rm{s.t.}}&\bm e_V^{(n)} =  W_V^\top \phi_V^{(n)} , \  n=1, \ldots, N,\\
     & \bm e_D^{(n)} =  W_D^\top  \phi_D^{(n)} , \ n=1, \ldots, N, 
\end{array}
\end{equation}
where $\lambda, \rho_1, \rho_2>0$ are regularization constants.  {With the Lagrangian duality and
the KKT conditions, the dual optimization  of \eqref{eq:cca:lssvm}
leads to  solving a generalized eigenvalue problem: 
\(
\begin{bmatrix}
\begin{smallmatrix}
0_N & K_D \\
K_V & 0_N
\end{smallmatrix}
\end{bmatrix}
\begin{bmatrix}
\begin{smallmatrix}
\bm\alpha_V \\
\bm\alpha_D
\end{smallmatrix}
\end{bmatrix}
=
\lambda
\begin{bmatrix}
\begin{smallmatrix}
\rho_1 K_D + I_N & 0_N \\
0_N & \rho_2 K_V + I_N
\end{smallmatrix}
\end{bmatrix}
\begin{bmatrix}
\begin{smallmatrix}
\bm\alpha_V \\
\bm\alpha_D
\end{smallmatrix}
\end{bmatrix},
\)where $K_V,K_D\in\mathbb{R}^{N\times N}$ are the kernel matrices induced by  embeddings  $\{\bm h_{V}^{(n)}\}_{n=1}^N$ and $\{\bm h_{D}^{(n)}\}_{n=1}^N$ of the finetuned visual encoder and the given DINOv2 model, 
respectively, such that
$[K_V]_{ij}
=\kappa_V(\bm h_V^{(i)},\bm h_V^{(j)})=\langle\phi_V(\bm h_V^{(i)}),\phi_V(\bm h_V^{(j)})\rangle$
and
$[K_D]_{ij}=\kappa_D(\bm h_D^{(i)},\bm h_D^{(j)})
=\langle\phi_D(\bm h_D^{(i)}),\phi_D(\bm h_D^{(j)})\rangle$.
The dual variables $\bm\alpha_V,\bm\alpha_D\in\mathbb{R}^{N}$ are then obtained as the eigenvectors. In KCCA, the largest 
eigenvalues give the optimal correlations and their associated  eigenvectors corresponding to the  pair projection directions\cite{bach2002kernel,suykens2002lssvm}. 
}


\paragraph{{The 3-view KCCA Extension of Joint Alignment with Text Semantics.}} Under the KCCA framework, the two views can be flexibly extended with multiple ones, which are jointly optimized to maximize mutual correlations \cite{benton2019dgcca,chapman2024stochasticcca,suykens2002lssvm}. Our  2-view KCCA optimization to align with DINOv2 $f_{\rm{dino}}$  can be accordingly further developed  to jointly align with the pretrained text encoder $f_{\rm{txt}}$, i.e., the 3rd view, for preserving image-text pairing semantics in the pretrained CLIP. In this way, the enhancement of visual representations and the maintenance of image-text retrievals are simultaneously considered under a unified optimization framework based on KCCA, leading to the following objective:
\begin{align}
J = \frac{1}{N}\sum_{n=1}^{N}
\Bigl(
\rho \Bigl(
{\bm{e}_V^{(n)\top}\bm{e}_D^{(n)}}
+
{\bm{e}_V^{(n)\top}\bm{e}_T^{(n)}}
+
{\bm{e}_D^{(n)\top}\bm{e}_T^{(n)}}
\Bigr)
-
\frac{1}{2}\sum_{v\in\mathcal{V}}
\Bigl(
\gamma_v \|\bm{e}_v^{(n)}\|_2^2
+
{\rm{Tr}}(W_v^\top W_v)
\Bigr)
\Bigr)
\label{eq:three_view_primal}
\end{align}
{with
$\bm e_v^{(n)}=W_v^\top\phi_v^{(n)}$
for  $v\in\mathcal{V}$, where the mutual correlations between these 3 views are maximized. Similarly, the dual problem of \eqref{eq:three_view_primal} also leads to   a generalized eigenvalue problem \cite{suykens2002lssvm}.}

\subsection{An End-to-End Optimization Scheme}
\label{sec:end_to_end_optimization}
{In optimization, the parameters $\bm \theta$ of the CLIP visual encoder $\hat{f}_{\rm{img}}$ are finetuned and the parameters in our \ours are also optimized. KCCA commonly solves a  generalized eigenvalue problem in the dual form. However, in this post-training process,
it is impractical to  repeatedly solve a  generalized eigenvalue problem during each iteration of finetuning the network parameters $\bm \theta$ of the CLIP visual encoder $\hat{f}_{\rm{img}}$. Thus, we resort to the optimization of \ours in the primal. In the following Lemma \ref{lem:kkt_objective}, we prove the zero-value property of the primal optimization objective $J$ in \eqref{eq:three_view_primal}, when it is evaluated at  stationary solutions with the KKT conditions, as proved in Appendix~\ref{app:kkt}.
\begin{lemma}
[A zero-value objective with stationary solutions] 
\label{lem:kkt_objective} 
Given  $\{\phi_v^{(n)}, v\in\mathcal{V}\}_{n=1}^{N}$, the stationary solutions under KKT conditions lead to the zero-value objective $J^\star=0$ in  primal optimization \eqref{eq:three_view_primal}.
\end{lemma} 
With the property in Lemma~\ref{lem:kkt_objective}, rather than solving an eigenvalue problem with kernels  in the
dual, we leverage this zero-value property, necessary to the stationary solution under the KKT conditions, to construct a practical end-to-end objective by minimizing the primal residual towards zero, i.e., $\min \sqrt{J^2}$. Similar to KUEA and KALE \cite{gong2025kuea,pawlowicz2026kale},
we keep the regularizer that penalizes an $\ell_2$ distance between the output embeddings of the finetuned visual encoder $\hat{f}_{\rm{img}}$ and the  pretrained  one $f_{\rm{img}}$:
\begin{equation} \label{eq:end_to_end_objective} \min_{\bm \theta,\{W_v\}_{v\in\mathcal{V}}} \mathcal{L} = \eta\sqrt{J^2} + \frac{1}{N} \sum\nolimits_{n=1}^{N} \| \hat f_{\mathrm{img}}(I^{(n)}) - f_{\mathrm{img}}(I^{(n)}) \|_2^2,
\end{equation}  with  regularization coefficient $\eta>0$.  Thus,  the optimization of our method is cast in an efficient end-to-end manner with  \eqref{eq:end_to_end_objective}, where the eigenvalue problem related to KCCA optimization is avoided. Note that since the projection weight matrices $W_v$  in \ours are low-dimensional, the additional computational overhead is  negligible compared to the CLIP  backbones.} 

\section{Experiments}
\label{sec:experiments}

In this section, {we evaluate our \ours on  fine-grained visual capabilities  and the pretrained image-text retrievals, and also investigate which specific components contribute to  performance gains.}

\subsection{Setups}
\label{sec:setups}
We adopt OpenAI CLIP ViT-L/14 and fine-tune its visual encoder on ImageNet-1K~\cite{deng2009imagenet}, while keeping the text encoder frozen.
A DINOv2 model with a ViT-L/14 backbone is used as the vision-centric supervisor.
We evaluate fine-grained visual capabilities on MMVP-VLM~\cite{tong2024eyes}, which contains visually distinct image pairs that CLIP maps to nearby representations across nine visual patterns. 
We also report zero-shot image--text retrieval on Flickr30K~\cite{young2014flickr30k} and MSCOCO~\cite{chen2015coco}. Full implementation details 
are provided in Appendix~\ref{app:implementation_details}.

\subsection{Results and Analysis}
\label{sec:main_results}

\paragraph{Main Results.} 
As shown in Table~\ref{tab:main_results}, our \ours improves the average
MMVP-VLM accuracy from $17.8$ to $25.9$, substantially outperforming DIVA and KUEA
by $2.9$ and $5.2$, respectively. In the meantime, we  improve the aggregate retrieval score from $72.7$ to $73.3$. 
{These results verify  the effectiveness of our \ours for enhancing  fine-grained visual
perception and retaining the image-text semantics in the pretrained CLIP.}
Complete per-dataset results are reported in
Appendix~\ref{app:full_transfer}.

\begin{table}[t]
\centering
   \vspace{-10pt}
\caption{Fine-grained perception and image–text retrieval performances (\%, $\uparrow$).}
  \vspace{-1pt}
\label{tab:main_results}
\small
\setlength{\tabcolsep}{2.8pt}
\renewcommand{\arraystretch}{1.06}
\resizebox{0.7\linewidth}{!}{%
\begin{tabular}{
lcccccccccc
@{\hspace{4pt}\vrule width 0.35pt\hspace{4pt}}
ccc
}
\toprule
& \multicolumn{10}{c}{{MMVP-VLM}}
& \multicolumn{3}{c}{{Image--Text Retrieval}} \\
\cmidrule(lr){2-11}
\cmidrule(lr){12-14}

Method
& \faCompass
& \faSearch
& \faSync
& \faSortNumericUp
& \faMapPin
& \faPalette
& \faCogs
& \faFont
& \faCamera
& Avg.
& I2T
& T2I
& Avg.\\
\midrule

CLIP~\cite{radford2021clip}
& 6.7
& 13.3
& 20.0
& 13.3
& 6.7
& 53.3
& 26.7
& 6.7
& 13.3
& 17.8
& 79.5
& 65.8
& 72.7 \\

DIVA~\cite{wang2025diva}
& \textbf{13.3}
& 20.0
& 40.0
& 13.3
& 6.7
& 53.3
& 33.3
& \textbf{13.3}
& 13.3
& 23.0
& \textbf{79.8}
& 65.8
& 72.8\\

KUEA~\cite{gong2025kuea}
& 0.0
& 20.0
& 26.7
& \textbf{20.0}
& \textbf{13.3}
& 33.3
& 33.3
& \textbf{13.3}
& \textbf{26.7}
& 20.7
& 78.7
& 67.1
& 72.9\\

\midrule

\textbf{\ours}
& \textbf{13.3}
& \textbf{26.7}
& \textbf{53.3}
& 13.3
& 0.0
& \textbf{60.0}
& \textbf{40.0}
& \textbf{13.3}
& 13.3
& \textbf{25.9}
& 79.0
& \textbf{67.7}
& \textbf{73.3} \\

\bottomrule
\end{tabular}%
}
   \vspace{-15pt}
\end{table}




\paragraph{Investigation on KCCA-based Alignment.} 

\begin{wraptable}{r}{0.46\columnwidth}
    \centering
    \vspace{-5pt}
    \caption{Ablation on  DINOv2 alignment.}
    \label{tab:transfer_diagnostic}
    \vspace{-1pt}
    \small
    \setlength{\tabcolsep}{6pt}
    \renewcommand{\arraystretch}{1.0}
    \resizebox{\linewidth}{!}{%
        \begin{tabular}{lc}
            \toprule
            Configuration & MMVP \\
            \midrule
            KUEA  & \textbf{20.7} \\
            KUEA (w/o DINOv2-alignment) & \textbf{20.7} \\
            KUEA (downweighting regularization) & 20.0 \\
            \midrule
            \ours  & \textbf{25.9} \\
            \ours (w/o DINOv2-alignment) & 20.7 \\
            \ours (shuffled pairing to DINOv2) & 21.5 \\
            \ours (DINOv2 replaced by random teacher) & 20.0 \\
            \bottomrule
        \end{tabular}%
    }
    \vspace{-10pt}
\end{wraptable}

Table~\ref{tab:transfer_diagnostic} revisits how DINOv2 supervision operates in KUEA and \ours. Specifically, removing the kernel-alignment loss  from KUEA leaves the MMVP-VLM accuracy unchanged at $20.7$; increasing its relative importance ratio  by further downweighting the image--text regularizer further reduces the alignment loss, but even slightly decreases the accuracy. These observations indicate that the element-wise kernel-matrix matching alone is insufficient to enhance fine-grained perception, motivating our subspace-correlation formulation. Visualization on losses during training are provied in Appendix~\ref{app:full_transfer}. In contrast, removing the KCCA-alignment loss from \ours lowers the MMVP-VLM accuracy from $25.9$ to $20.7$, while shuffling the CLIP--DINOv2 correspondences and replacing DINOv2 with a random teacher yields $21.5$ and $20.0$, respectively. These ablation results attribute our visual representation enhancement to the projection correlations  learned under DINOv2 supervision on subspaces with KCCA-based optimization.

\begin{wraptable}{r}{0.35\linewidth}
    \centering
    \vspace{-12pt}
    \caption{Ablation on multi-view.}
    \label{tab:view_ablation}
    \small
    \setlength{\tabcolsep}{4pt}
    \renewcommand{\arraystretch}{1.0}
    \begin{tabular}{lc}
        \toprule
        Configuration & MMVP \\
        \midrule
         3-view (w/ CLIP text)    & \textbf{25.9} \\
         3-view (w/ CLIP image)  & 20.7 \\
        2-view         & 23.0 \\
        \bottomrule
    \end{tabular}
    \vspace{-12pt}
\end{wraptable}
\paragraph{Investigation on Multi-view Setups.} We then present ablation studies on  the role of our deployed 3rd view  that leverages the embedding projections from the frozen CLIP text encoder. Table~\ref{tab:view_ablation} demonstrates that the 2-view KCCA-alignment to DINOv2 in Section \ref{sec:subspace_cca} readily improves the MMVP accuracy from $20.7$  to $23.0$, while incorporating the 3rd view  of the pretrained CLIP text encoder further brings distinctive improvement to $25.9$. Note that when taking the pretrained CLIP image representations as the 3rd view, the MMVP accuracy remains  at $20.7$. This  is in sharp contrast to our deployment with the CLIP text encoder being the 3rd view, which brings distinctive improvements, revealing the potential of the  complementary  text semantics  as done in our \ours for further enhancing visual representations.




\section{Conclusion and Outlook}
\label{sec:conclusion}

We revisit the fine-grained visual representation enhancement of CLIP under vision-centric supervision from DINOv2. Beyond element-wise matching of induced kernel matrices, we align representations in feature subspaces by maximizing projection correlations through KCCA, and extend this alignment into a 3-view formulation by incorporating embeddings from the frozen CLIP text encoder, namely \ours. With the Lagrangian duality and the KKT conditions, we derive an efficient end-to-end training scheme based on a necessary condition satisfied by stationary solutions. Experiments demonstrate substantial gains on MMVP-VLM while preserving image-text semantics, with varied ablations validating our alignment scheme upon KCCA and  its multi-view extension. 

Currently, our study is  conducted under  the CLIP-DINOv2 setups. It would be interesting and worthwhile to explore other VLM backbones, vision-centric teachers, and broader perceptual benchmarks in future work, which we hope could inspire further investigations on the potential of optimizing subspace projections for representation alignment  in VLMs and beyond.

\bibliographystyle{plainnat}
\bibliography{references_updated}

\newpage

\appendix
\section{Proof of Lemma~1: Zero-Value Property of the Three-View Primal}
\label{app:kkt}

This appendix proves the zero-value property stated in
Lemma~1. Given feature representations
$\{\phi_v^{(n)}\}_{n=1}^{N},v\in\mathcal{V}$, let
$\{W_v^\star,\bm e_v^{(n)\star}\},v\in\mathcal{V}$ denote a stationary solution to the
 primal problem of 3-view KCCA in Eq.~(2), satisfying the KKT conditions.
We use $J^\star$ to denote the value of the primal objective $J$
evaluated at such a stationary solution.

Recall that the three views are indexed by
$\mathcal{V}=\{V,D,T\}$ as subscripts. For each sample $n$ and view
$v\in\mathcal{V}$, the projection scores are formulated as
\begin{equation}
    \bm e_v^{(n)}
    =
    W_v^\top \phi_v^{(n)},
    \quad
    n=1,\ldots,N,
    \quad v\in\mathcal{V},
    \label{eq:app_constraint}
\end{equation}
where the centering operation is absorbed into the feature
representation $\phi_v^{(n)}$. The three-view primal objective in Eq.~(2) can be equivalently
written as
\begin{align}
J
={}&
\frac{\rho}{N}
\sum_{n=1}^{N}
\Big(
{\bm e_V^{(n)}}^\top \bm e_D^{(n)}
+
{\bm e_V^{(n)}}^\top \bm e_T^{(n)}
+
{\bm e_D^{(n)}}^\top \bm e_T^{(n)}
\Big)
\nonumber\\
&-
\frac{1}{2N}
\sum_{n=1}^{N}
\sum_{v\in\mathcal{V}}
\gamma_v
\left\|\bm e_v^{(n)}\right\|_2^2
-
\frac{1}{2}
\sum_{v\in\mathcal{V}}
\operatorname{Tr}
\left(
W_v^\top W_v
\right).
\label{eq:app_primal}
\end{align}

\begin{proof}
With the Lagrangian duality, the Lagrange multipliers
$\bm\alpha_v^{(n)}$ associated with the equality constraints in
Eq.~\eqref{eq:app_constraint} are introduced, leading to the following  Lagrangian function:
\begin{align}
\mathcal{G}
={}&
J
+
\frac{1}{N}
\sum_{n=1}^{N}
\sum_{v\in\mathcal{V}}
{\bm\alpha_v^{(n)}}^\top
\left(
W_v^\top\phi_v^{(n)}
-
\bm e_v^{(n)}
\right).
\label{eq:app_lagrangian}
\end{align}
Then, the KKT conditions are given with sufficing the primal feasibility together with stationarity of $\mathcal{G}$ w.r.t. the primal variables. We first consider stationarity with respect to the projection
matrices. For every $v\in\mathcal{V}$, we have 
\begin{equation}
\nabla_{W_v}\mathcal{G}
=
-W_v
+
\frac{1}{N}
\sum_{n=1}^{N}
\phi_v^{(n)}
{\bm\alpha_v^{(n)}}^\top
=
0,
\end{equation}
which gives
\begin{equation}
W_v
=
\frac{1}{N}
\sum_{n=1}^{N}
\phi_v^{(n)}
{\bm\alpha_v^{(n)}}^\top.
\label{eq:app_w_stationarity}
\end{equation}
Taking the Frobenius inner product of both sides of
Eq.~\eqref{eq:app_w_stationarity} w.r.t. $W_v$ and using primal
feasibility from Eq.~\eqref{eq:app_constraint},  it yields
\begin{align}
\operatorname{Tr}
\left(
W_v^\top W_v
\right)
&=
\frac{1}{N}
\sum_{n=1}^{N}
\operatorname{Tr}
\left(
W_v^\top
\phi_v^{(n)}
{\bm\alpha_v^{(n)}}^\top
\right)
\nonumber\\
&=
\frac{1}{N}
\sum_{n=1}^{N}
{\bm\alpha_v^{(n)}}^\top
W_v^\top\phi_v^{(n)}
\nonumber\\
&=
\frac{1}{N}
\sum_{n=1}^{N}
{\bm e_v^{(n)}}^\top
\bm\alpha_v^{(n)}.
\label{eq:app_norm_identity}
\end{align}

Next, with the stationarity w.r.t.  the projection scores, for each  view $v\in\mathcal{V}$, we have 
\begin{equation}
\nabla_{\bm e_v^{(n)}}\mathcal{G}
=
\frac{1}{N}
\left(
\rho
\sum_{u\in\mathcal{V},u\neq v}
\bm e_u^{(n)}
-
\gamma_v \bm e_v^{(n)}
-
\bm\alpha_v^{(n)}
\right)
=
0,
\end{equation}
leading to
\begin{equation}
\bm\alpha_v^{(n)}
=
\rho
\sum_{u\in\mathcal{V},u\neq v}
\bm e_u^{(n)}
-
\gamma_v \bm e_v^{(n)},
\quad
v\in\mathcal{V}.
\label{eq:app_e_stationarity}
\end{equation}

For compactness, we define
\begin{align}
\mathcal{S}
&=
\sum_{n=1}^{N}
\Big(
{\bm e_V^{(n)}}^\top\bm e_D^{(n)}
+
{\bm e_V^{(n)}}^\top\bm e_T^{(n)}
+
{\bm e_D^{(n)}}^\top\bm e_T^{(n)}
\Big),
\label{eq:app_corr}\\
\mathcal{Q}
&=
\sum_{n=1}^{N}
\sum_{v\in\mathcal{V}}
\gamma_v
\left\|\bm e_v^{(n)}\right\|_2^2.
\label{eq:app_quad}
\end{align}
By multiplying Eq.~\eqref{eq:app_e_stationarity} by
${\bm e_v^{(n)}}^\top$ and summing over the three views,  we have
\begin{align}
\sum_{v\in\mathcal{V}}
{\bm e_v^{(n)}}^\top\bm\alpha_v^{(n)}
={}&
\rho
\sum_{v\in\mathcal{V}}
\sum_{\substack{u\in\mathcal{V}\\u\neq v}}
{\bm e_v^{(n)}}^\top\bm e_u^{(n)}
-
\sum_{v\in\mathcal{V}}
\gamma_v
\left\|\bm e_v^{(n)}\right\|_2^2
\nonumber\\
={}&
2\rho
\Big(
{\bm e_V^{(n)}}^\top\bm e_D^{(n)}
+
{\bm e_V^{(n)}}^\top\bm e_T^{(n)}
+
{\bm e_D^{(n)}}^\top\bm e_T^{(n)}
\Big)
-
\sum_{v\in\mathcal{V}}
\gamma_v
\left\|\bm e_v^{(n)}\right\|_2^2.
\label{eq:app_score_identity}
\end{align}
Then, by summing up Eq.~\eqref{eq:app_norm_identity} over
$v\in\mathcal{V}$ and  using
Eq.~\eqref{eq:app_score_identity}, it  gives
\begin{align}
\sum_{v\in\mathcal{V}}
\operatorname{Tr}
\left(
W_v^\top W_v
\right)
&=
\frac{1}{N}
\sum_{n=1}^{N}
\sum_{v\in\mathcal{V}}
{\bm e_v^{(n)}}^\top\bm\alpha_v^{(n)}
\nonumber\\
&=
\frac{1}{N}
\left(
2\rho\mathcal{S}
-
\mathcal{Q}
\right).
\label{eq:app_kkt_identity}
\end{align}

Finally, we  can substitute Eq.~\eqref{eq:app_kkt_identity} into
Eq.~\eqref{eq:app_primal}, yielding
\begin{align}
J^\star
&=
\frac{\rho}{N}\mathcal{S}
-
\frac{1}{2N}\mathcal{Q}
-
\frac{1}{2}
\left[
\frac{1}{N}
\left(
2\rho\mathcal{S}
-
\mathcal{Q}
\right)
\right]
\nonumber\\
&=
\frac{\rho}{N}\mathcal{S}
-
\frac{1}{2N}\mathcal{Q}
-
\frac{\rho}{N}\mathcal{S}
+
\frac{1}{2N}\mathcal{Q}
\nonumber\\
&=
0.
\end{align}
Therefore, with
$\{\phi_v^{(n)}:v\in\mathcal{V}\}_{n=1}^{N}$, every stationary
solution satisfying the KKT conditions of the three-view primal
problem gives a zero-value objective, i.e.,
\begin{equation}
    J^\star = 0,
\end{equation}
which completes the proof.
\end{proof}

\paragraph{Remarks on Lemma 1 and our loss function.}
Lemma~1 establishes $J^\star=0$ as a necessary condition satisfied
by stationary solutions of our \ours under the KKT
conditions. For the total training loss of Eq.~\eqref{eq:end_to_end_objective} in the main manuscript,
we then use $\sqrt{J^2}$
as the loss function associated with this zero-value property and
minimize it jointly with the 
regularization term, which avoids the eigenvalue problem typically solved in KCCA and makes the training remain under  the end-to-end  setups as in pretraining or post-training generic neural networks.

\section{More Detailed Numerical Results}
\label{app:full_transfer}

This section complements the main results with complete zero-shot
image-text retrieval results and additional evaluations on zero-shot classification performance. Together, these evaluations provide a broader assessment of whether post-training preserves CLIP's image-text semantics and
zero-shot capabilities.

\paragraph{Zero-shot image-text retrieval.}
Table~\ref{tab:full_retrieval} reports Recall@$1/5/10$ for
image-to-text (I2T) and text-to-image (T2I) retrieval on Flickr30K and
MSCOCO. Our \ours achieves the highest overall retrieval average of $73.35$,
compared with $72.66$ for the original CLIP model. The improvements are
most pronounced in text-to-image retrieval, particularly on MSCOCO,
while the image-to-text performance remains competitive across both
datasets.

\begin{table}[t]
    \caption{
    Complete zero-shot image--text retrieval results on Flickr30K and
MSCOCO. All methods use OpenAI CLIP ViT-L/14 and DINOv2.
    }
    \label{tab:full_retrieval}
    \centering
    \small
    \setlength{\tabcolsep}{3.0pt}
    \renewcommand{\arraystretch}{1.08}

    \begin{tabular}{lccccccccccccc}
        \toprule
        \multirow{3}{*}{Method}
        & \multicolumn{6}{c}{Image-to-Text Retrieval}
        & \multicolumn{6}{c}{Text-to-Image Retrieval}
        & \multirow{3}{*}{\shortstack{Average\\Retrieval}} \\
        
        \cmidrule(lr){2-7}\cmidrule(lr){8-13}

        & \multicolumn{3}{c}{Flickr30K}
        & \multicolumn{3}{c}{MSCOCO}
        & \multicolumn{3}{c}{Flickr30K}
        & \multicolumn{3}{c}{MSCOCO}
        & \\

        \cmidrule(lr){2-4}
        \cmidrule(lr){5-7}
        \cmidrule(lr){8-10}
        \cmidrule(lr){11-13}

        & R@1 & R@5 & R@10
        & R@1 & R@5 & R@10
        & R@1 & R@5 & R@10
        & R@1 & R@5 & R@10
        & \\
        \midrule

        CLIP~\cite{radford2021clip}
        & 77.7 & 94.6 & 98.1
        & \textbf{50.0} & 73.7 & \textbf{82.9}
        & 59.8 & 83.5 & 89.3
        & 34.3 & 58.9 & 69.1
        & 72.66 \\

        DIVA~\cite{wang2025diva}
        & 78.7 & 94.8 & \textbf{98.3}
        & 49.9 & \textbf{74.1} & 82.8
        & 59.8 & 83.4 & 89.4
        & 34.2 & 58.8 & 69.0
        & 72.77 \\

        KUEA~\cite{gong2025kuea}
        & 77.6 & 93.8 & 98.0
        & 48.3 & 73.3 & 81.4
        & 61.8 & \textbf{85.4} & \textbf{90.8}
        & 35.2 & 59.7 & 69.9
        & 72.93 \\

        \textbf{Ours}
        & \textbf{79.1} & \textbf{95.2} & 97.6
        & 48.1 & 72.1 & 81.8
        & \textbf{62.5} & 84.9 & 90.5
        & \textbf{36.0} & \textbf{61.0} & \textbf{71.4}
        & \textbf{73.35} \\

        \bottomrule
    \end{tabular}
\end{table}

\paragraph{Zero-shot classification.}
To assess whether \ours preserves CLIP's zero-shot recognition
capability, Table~\ref{tab:full_classification} compares the original
CLIP model and \ours across nine classification datasets. Despite
variations across individual datasets, \ours achieves an average
accuracy of $68.06$, closely matching the $68.08$ of the original CLIP
model. The nearly unchanged average indicates that \ours retains CLIP's overall zero-shot classification performance.

\begin{table}[t]
    \caption{ Zero-shot classification accuracy of the pretrained CLIP model and our
post-trained model.}
    \label{tab:full_classification}
    \centering
    \small
    \setlength{\tabcolsep}{5.0pt}
    \renewcommand{\arraystretch}{1.08}

    \begin{tabular}{l|ccccccccc|c}
        \toprule
        \multirow{2}{*}{Model}
        & \multicolumn{9}{c|}{Zero-shot Datasets}
        & \multirow{2}{*}{\shortstack{Average\\Zero-shot}} \\
        \cmidrule(lr){2-10}

        & \rotatebox{70}{CIFAR10}
        & \rotatebox{70}{CIFAR100}
        & \rotatebox{70}{Caltech101}
        & \rotatebox{70}{FER2013}
        & \rotatebox{70}{OxfordPets}
        & \rotatebox{70}{DTD}
        & \rotatebox{70}{RESISC45}
        & \rotatebox{70}{EuroSAT}
        & \rotatebox{70}{ImageNet-O}
        & \\
        \midrule

        CLIP~\cite{radford2021clip}
        & \textbf{95.20}
        & \textbf{75.49}
        & 83.83
        & 50.00
        & \textbf{93.21}
        & \textbf{55.21}
        & \textbf{64.90}
        & \textbf{62.65}
        & 32.25
        & \textbf{68.08} \\

        \textbf{Ours}
        & 93.28
        & 72.94
        & \textbf{84.17}
        & \textbf{54.50}
        & 89.21
        & 53.94
        & 60.95
        & 57.48
        & \textbf{46.05}
        & 68.06 \\

        \bottomrule
    \end{tabular}
\end{table}

\paragraph{Visualization on Losses during Training}
To complement the ablation results in Table~\ref{tab:transfer_diagnostic}, we visualize the training losses of the alignment and regularization objectives for KUEA and \ours in Fig.~\ref{fig:loss_dynamics}. For KUEA, the kernel-alignment loss oscillates within a relatively stable range throughout training and does not exhibit a clear downward trend. This optimization behavior is consistent with our ablation result that removing the KUEA alignment term leaves the MMVP performance unchanged. Together, these observations suggest that the kernel-alignment objective provides limited effective contribution under the considered training setup, motivating us to revisit the alignment formulation.
For comparison, we also report the training losses of \ours under the KCCA-based alignment objective introduced in Sec.~\ref{sec:method}. As shown in Fig.~\ref{fig:loss_dynamics}, the alignment loss decreases rapidly during the early stage of training and then gradually converges to a low and relatively stable range.

\begin{figure}[t]
    \centering

    \begin{subfigure}[t]{0.485\textwidth}
        \centering
        \includegraphics[width=\linewidth]{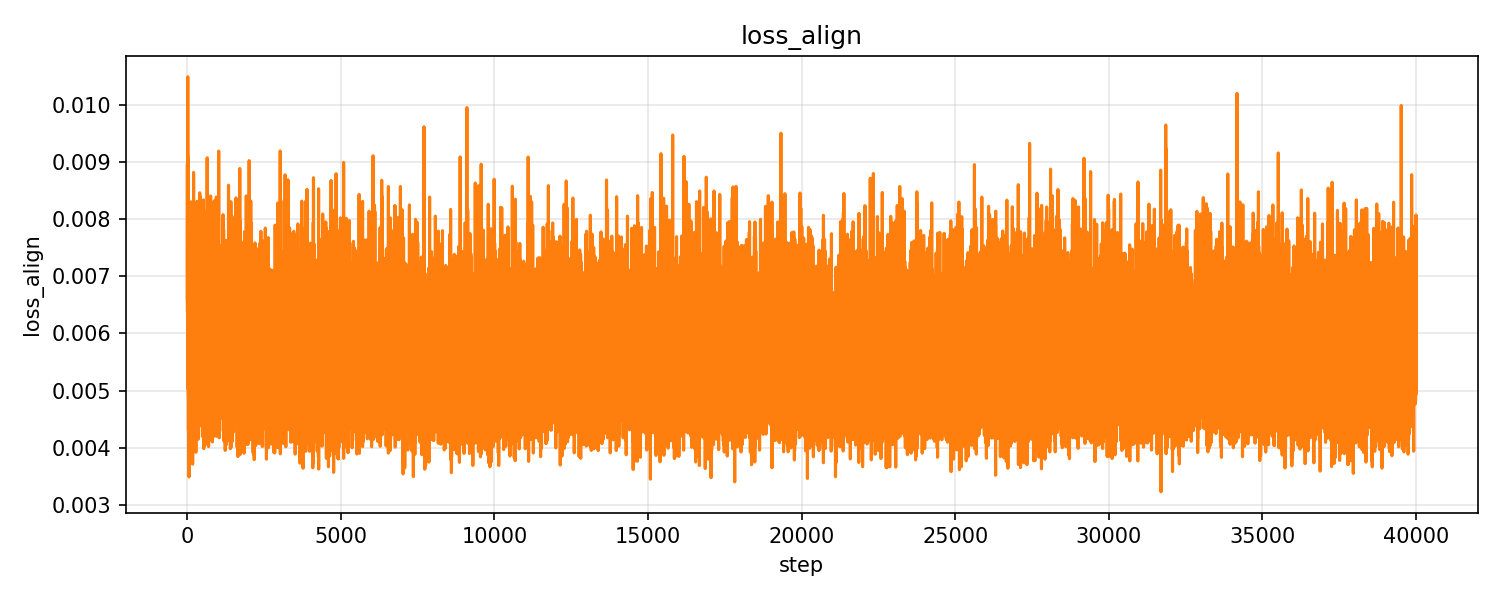}
        \caption{KUEA: alignment loss.}
        \label{fig:kuea_align_loss}
    \end{subfigure}
    \hfill
    \begin{subfigure}[t]{0.485\textwidth}
        \centering
        \includegraphics[width=\linewidth]{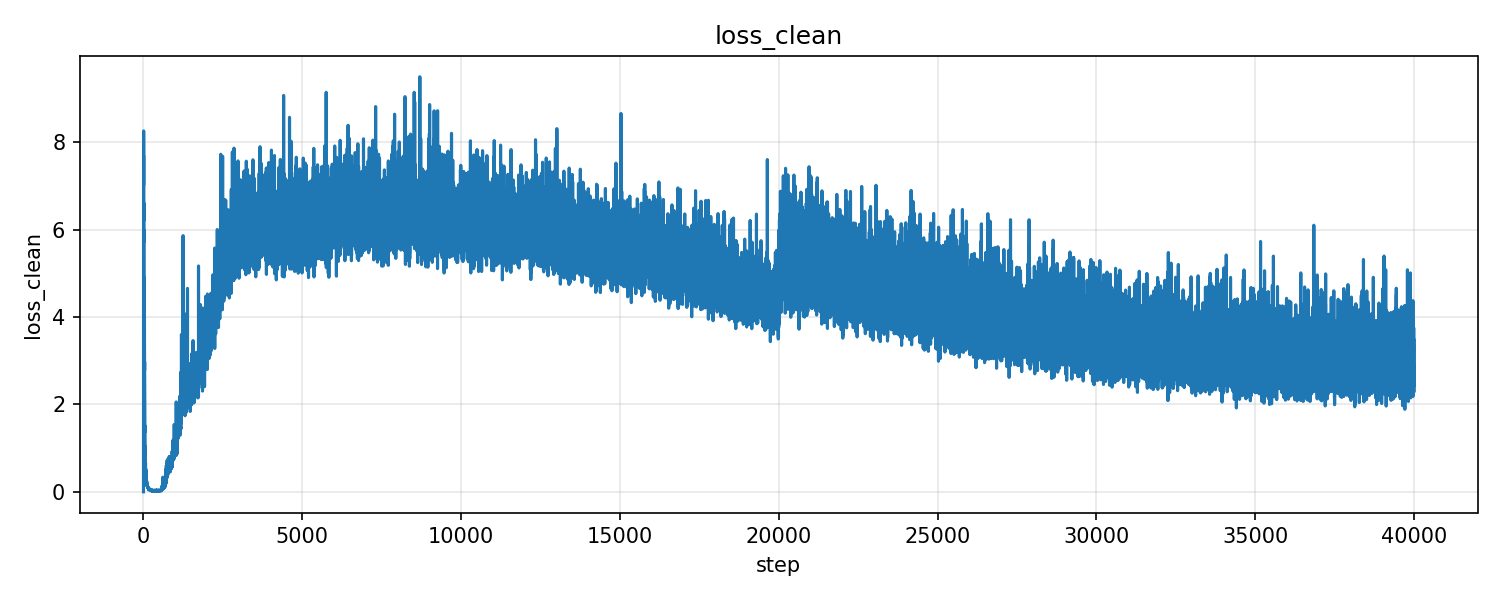}
        \caption{KUEA: regularization loss.}
        \label{fig:kuea_reg_loss}
    \end{subfigure}

    \vspace{0.6em}

    \begin{subfigure}[t]{0.485\textwidth}
        \centering
        \includegraphics[width=\linewidth]{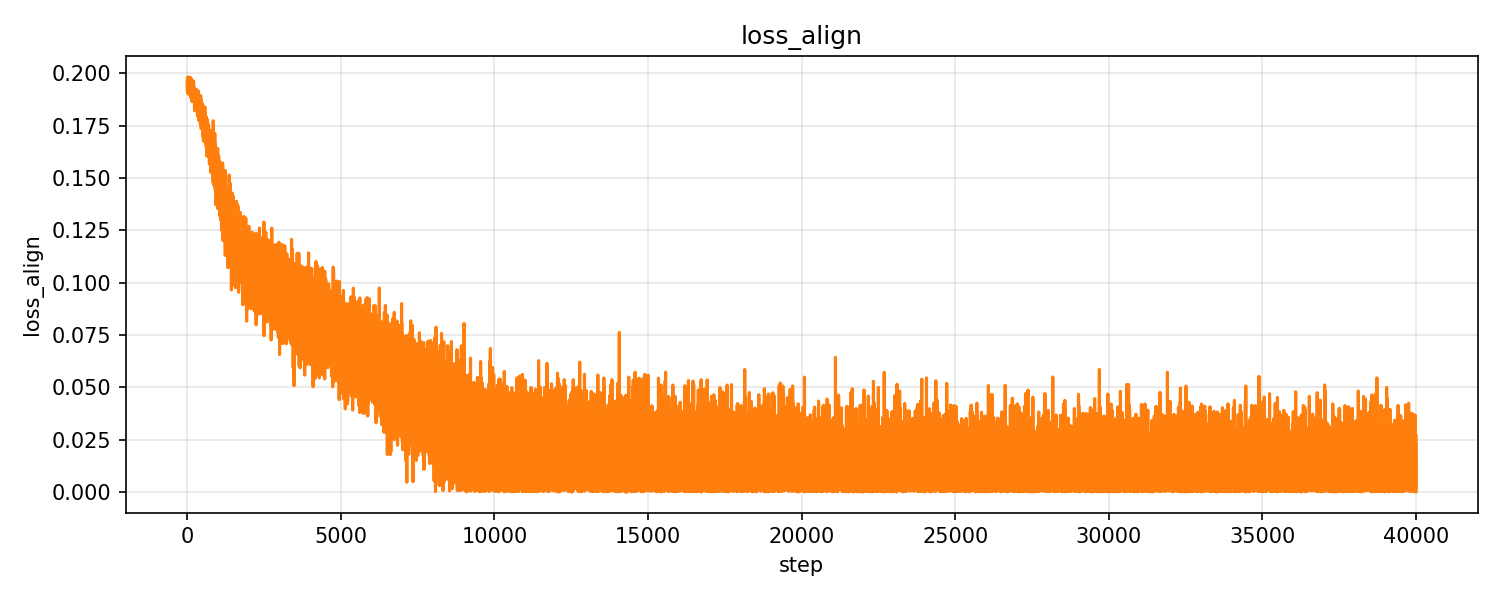}
        \caption{3vKCCA: alignment loss.}
        \label{fig:3vkcca_align_loss}
    \end{subfigure}
    \hfill
    \begin{subfigure}[t]{0.485\textwidth}
        \centering
        \includegraphics[width=\linewidth]{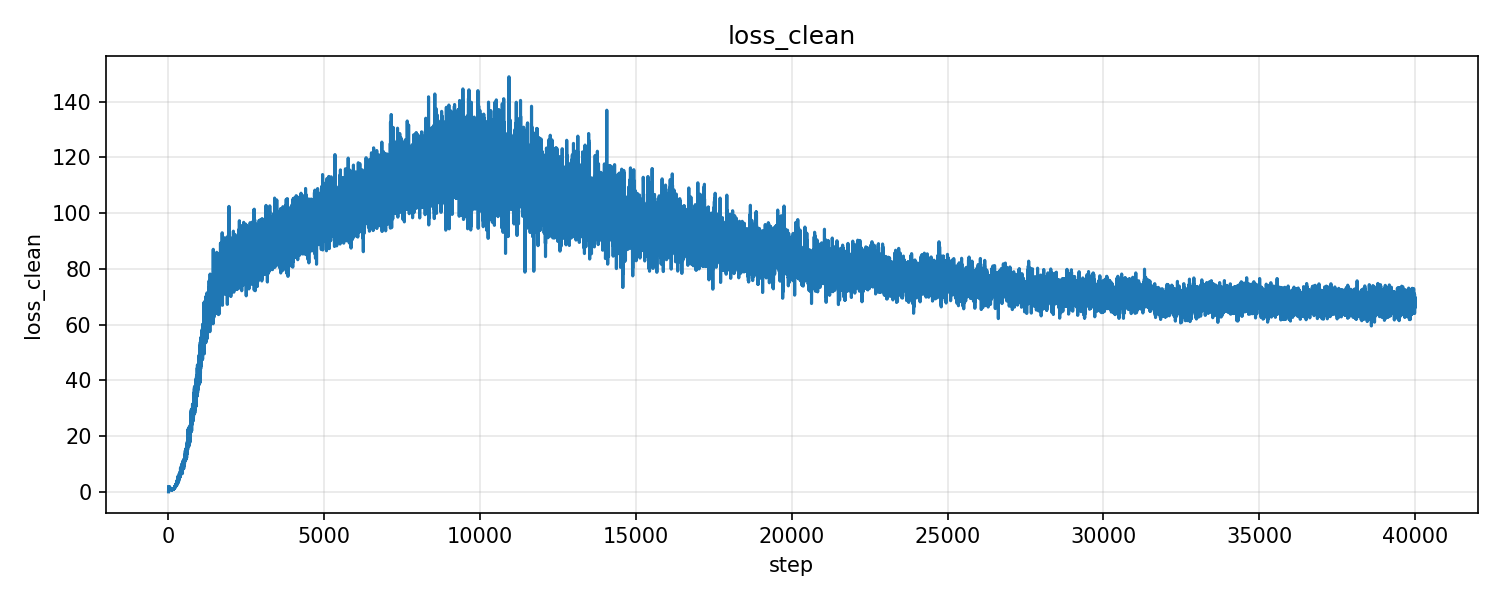}
        \caption{3vKCCA: regularization loss.}
        \label{fig:3vkcca_reg_loss}
    \end{subfigure}

    \caption{
    \textbf{Training dynamics of KUEA and \ours.}
    The left and right columns show the alignment and regularization losses, respectively, for KUEA and \ours. In our implementation, the alignment and regularization losses are recorded as \texttt{loss\_align} and \texttt{loss\_clean}, respectively.
    }
    \label{fig:loss_dynamics}
\end{figure}

\section{Implementation Details}
\label{app:implementation_details}

\paragraph{Feature Processing.}
For each view $v \in \mathcal{V}=\{V,D,T\}$, we construct the feature representation $\phi_v^{(n)}$ by applying sample-wise $\ell_2$ normalization followed by minibatch centering to the encoder embedding $\bm h_v^{(n)}$. Specifically, we conduct
\begin{equation}
\phi_v^{(n)}
=
\frac{\bm h_v^{(n)}}{\|\bm h_v^{(n)}\|_2}
-
\frac{1}{B}
\sum_{m \in \mathcal{B}}
\frac{\bm h_v^{(m)}}{\|\bm h_v^{(m)}\|_2},
\end{equation}
where $\mathcal{B}$ denotes the current minibatch of size $B$.
The centering operation is performed separately for each view. The resulting features are therefore normalized at the sample level and centered within the current minibatch before being projected by $W_v$.

\paragraph{Practical Implementations in Optimization.}
During end-to-end training, we apply a post-update normalization to the canonical projection matrices associated with the two visual views. Specifically, for
\begin{equation}
W_v=[\bm w_{v,1},\ldots,\bm w_{v,r}],
\qquad v\in\{V,D\},
\end{equation}
after each optimizer update, we perform
\begin{equation}
\bm w_{v,k}
\leftarrow
\frac{\bm w_{v,k}}{\left\|\bm w_{v,k}\right\|_2},
\qquad
v\in\{V,D\},
\quad
k=1,\ldots,r,
\end{equation}
where $r$ is the projection dimension. This column-wise normalization controls the scale of the canonical directions in the alignment to DINOv2 during training.
This also helps keep the scale of the two visual projection directions stable during training. We do not apply this  normalization to $W_T$, leaving the text-view projection unconstrained to retain scale flexibility in cross-modal correlation optimization.

For numerical stability, we implement the alignment objective as
$
\mathcal{L}_{\mathrm{align}}
=
\sqrt{J^2+\epsilon},$
where $J$ is the three-view primal objective defined in the main text with $\epsilon=10^{-8}$. The small constant $\epsilon$ stabilizes both the loss value and its gradient when $J$ approaches zero, while preserving $J=0$ as the zero-residual target. The regularization term follows the definition in the main text.
The complete training objective is then cast as:
\begin{equation}
\mathcal{L}
=
\eta\sqrt{J^2 + \epsilon}
+
\mathcal{L}_{\mathrm{reg}},
\end{equation}
with
$
\mathcal{L}_{\mathrm{reg}}
=
\frac{1}{N}
\sum_{n=1}^N
\left\|
\hat{f}_{\mathrm{img}}\!\left(I^{(n)}\right)
-
f_{\mathrm{img}}\!\left(I^{(n)}\right)
\right\|_2^2.$

\paragraph{Training Configuration}
We set the pairwise correlation weight to $\rho=5.0$ and use
$\gamma_V=\gamma_D=\gamma_T=1.0$ for the three views. The number of canonical components is set to $r=8$. For the practical training objective,  we weight the alignment term by
$\eta = 78.125$.
We train the model using AdamW with a learning rate of $10^{-5}$ and a weight decay of $10^{-4}$. Post-training is performed for $40{,}000$ optimization steps with a batch size of $64$, including $2{,}800$ warmup steps. All experiments use 2 × RTX 5090 GPUs.

\end{document}